\documentclass{article}

\usepackage{microtype}
\usepackage{tikz}
\usepackage{graphicx}
\usepackage{subfigure}
\usepackage{booktabs} 
\usepackage{pgfplots}

\usetikzlibrary{arrows}
\usetikzlibrary{arrows.meta}
\usetikzlibrary{shapes}
\tikzset{mynode/.style={inner sep=2pt,fill,outer sep=0,circle}}

\usepackage{hyperref}

\usepackage{amsmath,amssymb,xfrac,amsthm}
\newtheorem{lemma}{Lemma}
\newtheorem{theorem}{Theorem}
\newtheorem*{theorem*}{Theorem}

\newtheorem{proposition}{Proposition}
\newtheorem{definition}{Definition}

\usepackage[accepted]{icml2020}

\icmltitlerunning{LazyIter: A Fast Algorithm for Counting Markov Equivalent DAGs and Designing Experiments}

\begin{document}
	
	\twocolumn[
	\icmltitle{LazyIter: A Fast Algorithm for Counting Markov Equivalent DAGs and Designing Experiments}
	



\begin{icmlauthorlist}
\icmlauthor{Ali AhmadiTeshnizi}{sharif}
\icmlauthor{Saber Salehkaleybar}{sharif}
\icmlauthor{Negar Kiyavash}{epfl}
\end{icmlauthorlist}

\icmlaffiliation{sharif}{Department of Electrical Engineering, Sharif University of Technology, Tehran, Iran}
\icmlaffiliation{epfl}{School of Management of Technology, Ecole Polytechnique F´ed´erale de Lausanne, Switzerland}

\icmlcorrespondingauthor{Ali AhmadiTeshnizi}{ali.ahmadi215@student.sharif.edu}
\icmlcorrespondingauthor{Saber Salehkaleybar}{saleh@sharif.edu}
\icmlcorrespondingauthor{Negar Kiyavash}{negar.kiyavash@epfl.ch}

\icmlkeywords{Machine Learning, ICML, Causality, Causal Inference, Active Learning, Passive Learning, Experiment Design}

\vskip 0.3in
]



\printAffiliationsAndNotice{} 

	\begin{abstract}
		The causal relationships among a set of random variables are commonly represented by a Directed Acyclic Graph (DAG), where there is a directed edge from variable $X$ to variable $Y$ if $X$ is a direct cause of $Y$. From the purely observational data, the true causal graph can be identified up to a Markov Equivalence Class (MEC), which is a set of DAGs with the same conditional independencies between the variables. The size of an MEC is a measure of complexity for recovering the true causal graph by performing interventions. We propose a method for efficient iteration over possible MECs given intervention results. We utilize the proposed method for computing MEC sizes and experiment design in active and passive learning settings. Compared to previous work for computing the size of MEC, our proposed algorithm reduces the time complexity by a factor of $O(n)$ for sparse graphs where $n$ is the number of variables in the system. Additionally, integrating our approach with dynamic programming, we design an optimal algorithm for passive experiment design. Experimental results show that our proposed algorithms for both computing the size of MEC and experiment design outperform the state of the art.
	\end{abstract}
	
	\section{Introduction}
	Directed Acyclic Graphs (DAGs) are the most commonly used structures to represent causal relations between random variables, where a directed edge $X\rightarrow Y$ means that variable $X$ is a direct cause of variable $Y$. Conditional independencies between different variables can be inferred from observational data and consequently, the ground truth graph is identified up to Markov Equivalence Class (MEC) \cite{10.5555/1642718, Spirtes2000}.
	Unique identification of the ground truth DAG among the graphs in an MEC generally requires interventions on variables \cite{eberhardt2007interventions}. In some scenarios, interventions could be costly (for instance, in biological experiments), and therefore, selecting the optimal intervention target to learn the causal structure is of great interest \cite{eberhardt2005number,he2008active,eberhardt2012almost,Hauser_2014,shanmugam2015learning,kocaoglu2017experimental,ghassami2018budgeted,lindgren2018experimental,agrawal2019abcd}. Several metrics have been suggested in the literature for target selection \cite{he2008active,Hauser_2014,ghassami2018budgeted,agrawal2019abcd}. A good metric for measuring the effectiveness of an intervention is the number of remaining DAGs in an MEC after the intervention \cite{he2008active}. To use this metric for target selection, we must be able to efficiently count the number of DAGs in an MEC. 
	
	Some previous work
	used the clique tree representation of chordal graphs to divide the causal graph into smaller subgraphs, and perform counting on each subgraph separately \cite{Ghassami_2019,DBLP:conf/aaai/TalvitieK19}.
	 The main issue with this approach is the dependence on the maximum clique size which can result in $O(n!)$ operations in some cases where $n$ is the number of variables. \citet{Ghassami_2019} and \citet{DBLP:conf/aaai/TalvitieK19} used dynamic programming to count DAGs in an MEC. The best time complexity of these approaches is in the order of $O(2^n n^4)$. However, both approaches do not take advantage of sparsity if the graph is sparse. In some other work, the number of edges oriented after an intervention is proposed as the target selection metric \cite{Hauser_2014}. \citet{Hauser_2014} used the idea of conditioning on different edge orientations for edges connected to a single node to choose the optimal single-node intervention target. The time complexity of proposed algorithm depends on the size of largest clique in MEC which in the worst case is exponential. Recently, several work have been proposed for experiment design in passive and active learning settings which use the aforementioned metrics for target selection \cite{ghassami2018budgeted,kocaoglu2017experimental,agrawal2019abcd}. As such, it is desirable to efficiently compute the MEC size.

	In this paper, we propose ``LazyIter", a method for efficiently iterating over possible DAGs that we might get after an intervention on a single node. In this method, we start by setting a node as the root of the DAG and finding the corresponding essential graph resulting from intervening on this node. Subsequently, we take advantage of similarities between different candidate graphs to eliminate the recalculation of edge orientations and find other graphs just by reorienting a small subset of the edges.  We utilize this method to design algorithms for computing the size of MECs and solving the budgeted experiment design problem in active and passive settings. The main contributions of this paper are the following:
	\begin{itemize}

	    \item We propose an algorithm for computing the MEC size of a graph, which improves the time complexity by a factor of $O(n)$ in sparse graphs with respect to previous work \cite{DBLP:conf/aaai/TalvitieK19,Ghassami_2019}. Our experiments show that the algorithm outperforms previous work in dense graphs too. 
	    
	    \item In the active learning setting, we propose two algorithms for designing experiments for both metrics discussed earlier (number of edges and size of MEC). These algorithms are up to $O(n)$ times faster than the previous approaches \cite{he2008active,Hauser_2014}.
	    \item In the passive learning setting, we propose a dynamic programming algorithm for experiment design. To the best of our knowledge, this is the first efficient exact algorithm capable of finding the optimal solution in the passive learning setting. The most closely related work is an approximation algorithm presented in \citep{Ghassami_2019}, which has a considerably higher computational complexity.
	    \end{itemize}

	The paper is organized as follows: First, we discuss the terminology and preliminaries in Section 2. Then, in Section 3, we explain our iteration approach and prove its correctness. In Sections 4 and 5, we apply this approach to design algorithms for computing the MEC size and experiment design and also analyze their complexities. Finally, in Section 6, we demonstrate the efficiency of these algorithms by evaluating them on a diverse set of MECs. 
	
	\section{Preliminaries}
	\subsection{Graph Terminology}
	A graph $G(V,E)$ is represented with a set of nodes $V$ and a set of edges $E$, where each edge is a pair $(a,b)$ such that $a,b \in V$. We say there is an undirected edge between nodes $a$ and $b$ if both $(a,b)\in E,(b,a) \in E$, and say there is a directed edge from $a$ to $b$ if $(a,b) \in E,$ and $(b,a) \notin E$. A directed (undirected) edge is denoted with $a \rightarrow b \in G$ (or $a-b \in G$). We also use $(a,b) \in E$ and $(a,b) \in G$ subsequently. The set of all directed edges of $G$ is denoted by $Dir(G)$, and the number of directed edges in $G$ is denoted by $|Dir(G)|$.
	A graph is called undirected (directed) if all of its edges are undirected (directed), and is called partially directed if it has both undirected and directed edges.
	The induced subgraph $G[S]$ is the graph with node set $S$ and with edge set containing all of the edges in $E$ that have both endpoints in $S$. Union of graphs $G_1(V,E_1), G_2(V,E_2), ..., G_k(V,E_k)$ with the same set of nodes is defined as $\bigcup_{i=1}^k G_i = G(V, \bigcup_{i=1}^k E_i)$. For convenience, we may use $G$ and $V$ interchangeably. Two graphs are equal if they have the same set of nodes and the same set of edges.
	
	A path is a sequence of nodes $x_1, x_2, x_3, \dots, x_k$ such that $\forall  1\leq i < k: (x_i, x_{i+1}) \in E$. A cycle is a sequence of nodes $x_1, x_2, ..., x_k$ such that $\forall  1\leq i \leq k: (x_i, x_{i+1}) \in E$ where $x_k=x_1$. A path (cycle) is called directed if all of its edges are directed. Node $x$ is called a descendant of node $v$ if there is directed path from $v$ to $x$, and there are no directed paths from $x$ to $v$ in the graph. A chain graph is a graph with no directed cycles, and a chain component is a connected component of a chain graph after removing all its directed edges. An undirected graph is chordal if for every cycle of length four or more in it, there exists an edge which is not a part of the cycle but connects two nodes of the cycle to each other. 
	
	Let $G(V, E)$ be a partially directed graph. The skeleton of $G$ is an undirected graph that we get by replacing all of the directed edges in $E$ by undirected edges. We say node $v\in V$ is separated from node $u$ by set $T \subset V$ if there is no path from $v$ to $u$ in the skeleton of $G[V\backslash T]$, and we call $T$, a $(v,u)$-separator in $G$ \footnote{Please note that the definition of separator here is different from the definition of d-separation in causal Bayesian networks. }. 
	We denote parents, children, and neighbors of node $v\in V$ by $pa_G(v)$, $ch_G(v)$, and $ne_G(v)$, respectively. A perfect elimination ordering (PEO) in a graph $G$ is an ordering of its vertices such that for every vertex $v$, $v$ and its neighbors prior to it in the ordering form a clique. A graph is chordal if and only if it has a perfect elimination ordering \cite{fulkerson1965}.
	
	\subsection{Causal Model}
	A causal DAG $D$ is a DAG with variables $V_1,\cdots,V_n$ where there is a directed edge from $V_i$ to $V_j$ if $V_i$ is a direct cause of $V_j$. A joint probability distribution $P$ over these variable satisfies Markov property with respect to $D$ if any variable is independent of its non-descendants given its parents. A Markov Equivalence Class (MEC) is a set of DAGs with the same Markov property. \citet{Verma_1992} showed that the graphs in an MEC have the same skeleton and the same set of v-structures (induced subgraphs of the form $a\rightarrow b \leftarrow c$). The essential graph of $D$ is defined as a partially directed graph $G(V, E)$ where $E$ is the union of all edge sets of the DAGs in the same MEC as $D$. An essential graph is necessarily a chain graph with chordal chain components \cite{hauser2011characterization}. \citet{Verma_1992} showed that having observational data, essential graph is obtainable by applying four rules (called "Meek" rules) consecutively on the graph, until no more rules are applicable. A \textit{valid} orientation of edges of a chain component is an orientation in which no cycles and no v-structures are formed. An intervention target $I \subseteq V$ is a set of nodes which we intervene on simultaneously. An intervention family $\mathcal{I}$ is a set of intervention targets. Intervention graph $D^{(I)}$ is the DAG we get from $D$ after removing all edges directed towards nodes in $I$. 
	\begin{definition}
	For a set of intervention targets $I$, two DAGs $D_1$ and $D_2$ are called $\mathcal{I}$-Markov Equivalent (denoted with $D_1 \sim_\mathcal{I} D_2$) if they are statistically indistinguishable under intervention targets in $\mathcal{I}$.
	\end{definition}
	\citet{hauser2011characterization} proved that two DAGs $D_1$ and $D_2$ are $\mathcal{I}$-Markov equivalent if and only if $D_1$ and $D_2$ have the same set of v-structures, and $D^{(I)}_1$ and $D^{(I)}_2$ have the same skeleton for every $I \in \mathcal{I} \cup \{\emptyset\}$.
	
	The $\mathcal{I}$-essential graph $\mathcal{E}_\mathcal{I}(D)$ of a DAG $D(V,E)$ is a partially directed graph with the node set $V$ and the edge set equal to the union of all edge sets of the DAGs which are $\mathcal{I}$-Markov equivalent with $D$. $\mathcal{I}$-MEC is defined as the set of all DAGs that are $\mathcal{I}$-Markov equivalent. 
	\begin{definition}
	For undirected chordal chain graph (UCCG) $G(V,E)$ and intervention family $\mathcal{I}$, the \textit{intervention result space} is defined as:
	$$\mathcal{IR}_{\mathcal{I}}(G) = \{\mathcal{E}_{\mathcal{I}}(D): D\in \textbf{D}(G)\},$$
	where $\textbf{D}(G)$ denotes the set of all DAGs inside MEC corresponding to $G$.
	\end{definition}
	 We use $MEC(\mathcal{E}_\mathcal{I}(D))$ to show the set of all DAGs in an $\mathcal{I}$-MEC. Throughout the paper, we assume UCCGs are chain components of observational essential graphs.

	\begin{figure*}
		\begin{center}
			\begin{tikzpicture}[scale=0.4,line width=1pt]
			
			\node[{shape=circle, text=black, minimum size=0.1em}] (v) at  (0,-8) {a};
			
			\tikzset{edge/.style = {->,> = latex'}}
			\node[{shape=circle,draw, fill=gray!40, text=black, minimum size=0.1em}] (v) at  (1,3) {v};
			\tikzset{vertex/.style = {thick, shape=circle,draw,minimum size=3em}}
			\node[vertex] (p) at  (-2,0) {$P$};
			
			\node[vertex] (c) at  (4,0) {$C_R$};
			\node[thick,  draw, ellipse, minimum width=4em, minimum height = 6em] (a) at  (-6,1) {$A_R$};
			\tikzstyle{every node}=[thick,  draw, ellipse, minimum width=10em, minimum height = 3em]
			\node[] (d)  at  (5,-4) {$D_R$};

			\tikzset{edge/.style = {line width=0.1em, -{Latex[length=3mm, width=0.5em]}}}
			
			\draw[edge, bend left] (p) to (v);	
			\draw[edge, bend left] (v) to (c);
			
			\draw[edge] (p) to (c);
			
			\draw[edge, bend right] (p) to (d);
			\draw[edge] (c) to (d);
			
			\draw[line width=0.1em] (p) to (a);
			\end{tikzpicture}
			\hspace{1cm}
			\begin{tikzpicture}[scale=0.5,line width=1pt]
			
			\node[{shape=circle, text=black, minimum size=0.1em}] (v) at  (0,-7) {b};	
			
			\tikzset{edge/.style = {->,> = latex'}}
			\node[{shape=circle,draw, fill=gray!40, text=black, minimum size=0.1em}] (v) at  (1,3) {v};
			\node[{shape=circle,draw, fill=gray!40, text=black, minimum size=0.1em}] (u) at  (0,-1.5) {u};
			\tikzset{vertex/.style = {thick, shape=circle,draw,minimum size=3em}}
			\node[vertex] (p) at  (-3,0) {$P$};
			\node[vertex] (m) at  (-3.5,-4) {M};
			
			\node[vertex] (c) at  (4,0) {$C_R\backslash\{u\}$};
			
			\node[thick,  draw, ellipse, minimum width=4em, minimum height = 6em] (a) at  (-6,1) {$A_R$};
			\tikzstyle{every node}=[thick,  draw, ellipse, minimum width=10em, minimum height = 3em]
			\node[] (d)  at  (5,-4) {$D_R\backslash M$};
			
			\tikzset{edge/.style = {line width=0.1em, -{Latex[length=3mm, width=0.5em]}}}
			
			\draw[edge, bend left] (p) to (v);	
			\draw[edge, bend left] (v) to (c);
			\draw[line width=0.1em] (p) to (u);
			\draw[edge] (u) to (v);
			\draw[edge] (u) to (c);
			\draw[edge] (u) to (v);
			\draw[edge] (u) to (d);
			
			\draw[edge] (p) to (c);
			
			\draw[edge, bend right] (p) to (d);
			\draw[edge] (c) to (d);
			
			\draw[line width=0.1em] (p) to (a);
			\draw[line width=0.1em] (p) to (m);
			\draw[line width=0.1em] (u) to (m);
			\draw[dashed, bend left, ->, shorten >= 1em, shorten <= 1.5em] (d) to (m);
			\draw[dashed, bend left, ->, shorten >= 0.5em, shorten <= 1em] (m) to (a);
			
			\end{tikzpicture}
		\end{center}
		\caption{a) A separation of $R = \mathcal{P}^P_v$ into different sets. b) Constructing $R' = \mathcal{P}^{P \cup \{u\}}_v$ from  $R$ by moving $u$ from $C_R$ to $P$ and $M$ from $D_R$ to $A_R$. An arrow between two sets/nodes means that any edge between them is directed in the corresponding direction. A straight line between two sets/nodes, means that any edge between them is undirected. Dashed lines show how nodes are moved from $D_R$ to $A_R$ upon the construction.} 
		\label{fig:model}
	\end{figure*}
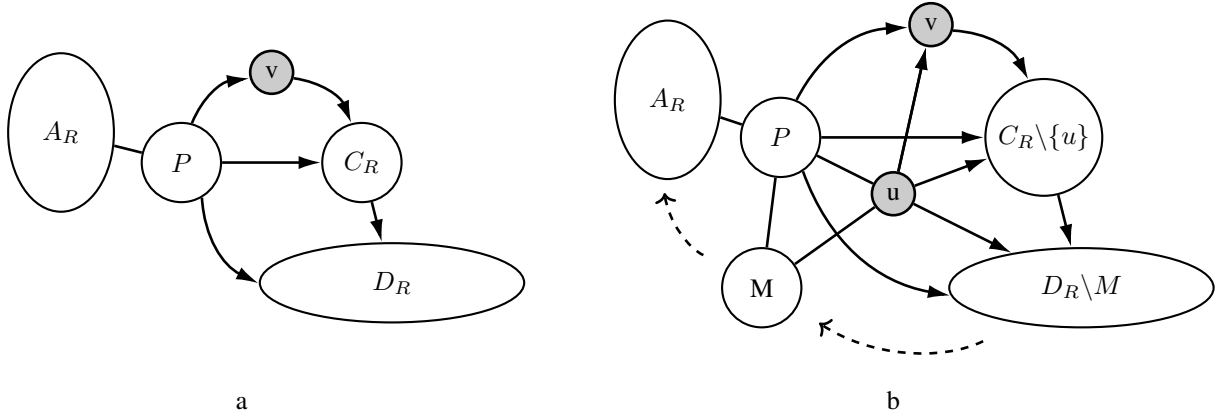
	
	\section{LazyIter}
	
	We first propose a method to select the best single-node intervention target in an essential graph. As $\mathcal{I}$-essential graph $\mathcal{E}_I(D)$ on DAG $D$ is a chain graph with undirected chordal chain components, it could be shown that knowing orientations of edges inside a component does not provide any information about the orientation of edges in other components \cite{hauser2011characterization}. \citet{Hauser_2014} showed that each chain component can be treated as an observational essential graph when it comes to intervening on the nodes (i.e., $\textbf{D}(G)$ is the same set of DAGs, whether $G$ is an observational essential graph or it is a chain component of an $\mathcal{I}$-essential graph). Consequently, we can restrict our attention to UCCGs. \citet{he2008active} presented a method to find $\mathcal{I}$-essential graph from intervention results when the intervention target is the root, which takes $O(n\Delta^2)$ operations. We will use this method in the next sections as a subroutine for computing the size of $\mathcal{I}$-essential graph whenever conditioning on edge orientations results in the intervention target becoming root.
	
	Let $G(V,E)$ be a UCCG and $\{v\}$ be a single-node intervention target on it. After the intervention, we will obtain an $\mathcal{I}$-essential graph $\mathcal{E}_{\{\{v\}\}}(D) \in \mathcal{IR}_{\{\{v\}\}}(G)$ based on the ground truth DAG $D$. The following theorem allows us to use parent set of $v$ for uniquely representing the resulting $\mathcal{I}$-essential graph:

	\begin{proposition}
		\label{thm:onetoone}
		Let $G(V,E)$ be a UCCG, $D \in MEC(G)$ be a DAG, and $v \in V$ be an arbitrary node. Then each $\mathcal{I}$-essential graph $\mathcal{E}_{\{\{v\}\}}(D)$ could be uniquely determined given the parent set of $v$ in $D$, and there is a one-to-one correspondence between sets $\{P \subseteq ne_G(v): P \text{ is a clique}\}$ and $\mathcal{IR}_{\{\{v\}\}}(G)$ .
	\end{proposition}
	The proof of this proposition as well as all other proofs are available in the supplementary material.
	The theorem suggests a way for iterating over $\mathcal{IR}_{\{\{v\}\}}(G)$: Iterate over all cliques in the neighborhood of $v$ and set each clique as the parent set of $v$ and then apply Meek rules to orient as many edges as possible \cite{Hauser_2014}. According to Proposition \ref{thm:onetoone}, the essential graph $\mathcal{E}_{\{\{v\}\}}(D)$ can be determined by $pa_D(v)$. Thus, we use the notation of $\mathcal{P}_v^P(G)$ to point to $\mathcal{E}_{\{\{v\}\}}(D)$ where $D\in \textbf{D}(G)$ is a DAG such that $pa_D(v) = P$.
	
	Let $R(V, E') = \mathcal{P}^P_v(G)$ be a possible single-node-intervention result on a UCCG $G(V, E)$. Setting aside the nodes in $P$, we divide the other nodes of $R$ into three distinct groups $C_R$, $A_R$, and $D_R$. $C_R$ is the set of children of $v$, and $C_R \cup P = ne_G(v)$. $A_R$ is the set of all nodes which are separated from $v$ by $P$, and $D_R$ is the set of all other nodes. We have:
	$$A_R =  \{a \in V \backslash ne_G(v): P \text{ is an } (a,v)\text{-separator in }G\}$$
	$$D_R = V\backslash (A_R \cup C_R \cup P).$$
	
	 See Figure \ref{fig:model}a for an illustration.
	 The following theorem states several key properties of the three proposed node groups.
	
	\begin{theorem}
		\label{thm:properties}
		Let $R=\mathcal{P}^P_v(G)$ be an $\mathcal{I}$-essential graph on a UCCG $G(V,E)$. The following statements hold:
		\begin{itemize}
			\item There are no edges in $G$ connecting a node in $A_R$ to a node in $C_{R} \cup D_{R} \cup \{v\}$.
			\item Every edge $(a,b)$ in ${R}$ where $a \in P$ and $b \in C_{R}$ is directed as $a\rightarrow b$.
			\item Every edge $(a,b)$ in ${R}$ where $a \in C_{R} \cup P$ and $b \in D_{R}$ is directed as $a\rightarrow b$.
			\item All of the edges in $R[A_{R} \cup P]$ are undirected. 
		\end{itemize}
	\end{theorem}
	
	\begin{algorithm*}[tb]
		\caption{LazyIter}
		\label{alg:LazyIter}

		\begin{algorithmic}[1] 
			\STATE \textbf{Input:} UCCG $G(V,E)$, Node $v\in V$ \\
			\STATE \textbf{Output:} $\mathcal{IR}_{\{\{v\}\}}(G)$
			\STATE  $\mathcal{L} \leftarrow \emptyset$
			\STATE Find $\mathcal{P}^{\emptyset}_v(G)$ by setting $v$ as the root of $G$ and orienting as much edges as possible.
			\STATE $Iter(\mathcal{P}^{\emptyset}_v(G),v)$
			\STATE \textbf{return $\mathcal{L}$} \newline ---------------------------------------------------------------------------------------------------------------------------------------------
			\STATE \textbf{function} $Iter(\mathcal{P}^{P}_v(G),v)$
			\STATE Add $\mathcal{P}^{P}_v(G)$ to $\mathcal{L}$.
			\FOR {$u \in C_R$} 
			\IF {$u$ is connected to all nodes in $P$}
			\STATE $NewGraph \longleftarrow \mathcal{P}^{P}_v(G)$
			\STATE $M \longleftarrow $ Set of all nodes separated from $v$ by $P \cup \{u\}$ in $G[V \backslash A_{\mathcal{P}^{P}_v(G)}]$
			\STATE Change direction of $v \rightarrow u$ to $u \rightarrow v$ in $NewGraph$
			\STATE In $NewGraph$, make all edges with both ends in $M \cup \{u\}$ undirected
			\STATE In $NewGraph$, make all edges connecting a node in $P$  to a node in $M\cup\{u\}$ undirected
			\STATE In $NewGraph$, direct all edges between $u$ and a node $c\in C_R$ as $u\rightarrow c$
			\STATE In $NewGraph[C_{\mathcal{P}^{P}_v(G)}]$, orient edges using Meek rules until no more undirected edges are orientable
			\STATE $Iter(NewGraph, v)$
			\ENDIF
			\ENDFOR
			
		\end{algorithmic}
	\end{algorithm*}
	
	Now we show that direction of many edges in $\mathcal{P}^P_v(G)$ stay intact when we change the parent set $P$ slightly, and therefore if we already know direction of edges in an $\mathcal{I}$-essential graph, we can find the direction of edges in other $\mathcal{I}$-essential graphs by reorienting just a small fraction of the edges.

	Assume we are given $R$ and we want to find $R' = \mathcal{P}^{P \cup \{u\}}_v(G)$ where $u \in C_R$, and $G[P \cup \{u\}]$ is a clique. Note that the skeleton of both $R'$ and $R$ is $G$, and they only differ in the direction of some edges.
	
	It is easy to see $A_R \subseteq A_{R'}$ as every node which is separated from $v$ by $P$ is also separated from $v$ by $P \cup \{u\}$. Moreover we know that $A_R \cup D_R = A_{R'} \cup D_{R'}$ as both of them represent the set of nodes in $V \backslash (ne_G(v) \cup \{v\})$. Consequently, we have $D_{R'} \subseteq D_R$ and $D_R \backslash D_{R'} = A_{R'} \backslash A_{R} =M$.
	We construct $R'$ from $R$ by moving $u$ from children to the parents and $M$ from $D_R$ to $A_R$, and then reorienting some specific edges as we explain. We have:
	$$M = \{a \in D_R:\ P\cup\{u\} \text{ is an } (a,v)\text{-separator in } G\}$$
	$$C_{R'} = C_R \backslash \{u\}, A_{R'} = A_R \cup M, D_{R'} = D_R \backslash M.$$
	
	Applying the first statement of Theorem \ref{thm:properties} to $R'$, we conclude that there are no edges between $M$ and $D_R \backslash M$ in G (as $M\subseteq A_{R'}$ and $D_R\backslash M = D_{R'})$. 
	The third statement of Theorem \ref{thm:properties} implies that in $R'$, any edge between $(P \cup \{u\}) \cup C_{R'} = P \cup C_R$ and $D_{R'} = D_R \backslash M$ is directed towards the node in $D_R \backslash M$. The same thing is true in $R$, as we have $D_R \backslash M \subseteq D_R$. This means any edge in $R[D_R \backslash M]$ which is directed by applying Meek rules, can be similarly directed in $R'[D_R \backslash M]$, and therefore $R'[D_R \backslash M] = R[D_R\backslash M]$.
	Moreover, we can say that $R'[A_R\cup P] = R[A_R\cup P]$ because both are undirected graphs on the same skeleton. Using the fourth statement of Theorem \ref{thm:properties}, we can infer that all of edges in $R'[M \cup \{u\}]$ are undirected, as $M \cup \{u\} \subseteq A_{R'}$. The same is true for edges with one end in $M\cup \{u\}$ and the other end in $P$.
	Finally, by the second statement of Theorem \ref{thm:properties}, all of the edges in $R'[C_R]$ which are connected to $u$ are directed away from $u$. 
	This means we can find the orientation of edges in $R'$ by executing the following three steps on $R$:
	\begin{enumerate}
		\item Obtain the set $M$ by finding nodes in $G[V\backslash A_R]$ which are separated from $v$ by $(P \cup \{u\})$. If we execute a breadth first search (BFS) in $G[V\backslash (P \cup \{u\})]$ with $v$ as root, the nodes which are not observed in the BFS constitute $A_{R'}$. By removing nodes of $A_R$ from $A_{R'}$ we will get the set $M$. This will take $O(n + m) = O(n + n\Delta) = O(n\Delta)$ operations, where $n,m$, and $\Delta$ are the number of variables, the number of the edges, and the maximum degree of the graph respectively.
		\item Remove the directions of all edges inside $R[M\cup\{u\}]$ and all edges between $M\cup\{u\}$ and $P$. This could be done in $O(n\Delta)$ operations.
		\item Direct all edges $u-x$ in $R[C_R]$ as $u\rightarrow x$, and apply Meek rules on $R[C_R]$ to find $R'[C_R\backslash\{u\}]$. This could be done in $O(\Delta^3)$ operations \cite{he2015counting}, as we have $|C_R|\leq \Delta$.		
	\end{enumerate}
	The procedure for finding $IR_{\{\{v\}\}}(G)$ is given in Algorithm \ref{alg:LazyIter}. In order to find $IR_{\{\{v\}\}}(G)$, first we obtain $\mathcal{P}^{\emptyset}_v(G)$ by setting $v$ as the root of the graph and directing edges based on Meek rules in $O(n\Delta^2)$ operations \cite{he2015counting}. Then we initiate $\mathcal{L}$ as an empty set and call $LazyIter(\mathcal{P}^{\emptyset}_v(G), v)$ which will add all desired $\mathcal{I}$-essential graphs to set $\mathcal{L}$ (for finding $\mathcal{P}^{\emptyset}_v(G)$ we can use the algorithm presented in \cite{he2015counting} which needs $O(n\Delta ^2)$ operations). The algorithm will call itself recursively $O(2^\Delta)$ times, and the three mentioned operations are executed in each call in order to find the new $\mathcal{I}-$essential graph corresponding to the new parent set. When the execution is completed, $\mathcal{L}$ will contain the list of all obtainable $\mathcal{I}$-essential graphs. The complexity of the algorithm is $O(n\Delta^2 + 2^\Delta(n\Delta + \Delta^3)) = O(2^\Delta(n\Delta + \Delta^3))$. The first step is executed in line 12 of Algorithm \ref{alg:LazyIter}, the second step is executed in lines 14 and 15, and the last step is executed in lines 16 and 17.

	\section{Computing size of MEC}
	We count the number of DAGs inside an MEC by partitioning them into $\mathcal{I}$-Markov equivalence classes.  
	
	\begin{lemma}
	    \label{lemma:sigmapi}
	    	Let $G(V,E)$ be a UCCG and $\mathcal{I}$ be an arbitrary intervention family. Then we have:
	    	\begin{equation*}
    	        |MEC(G)| = \sum_{R \in \mathcal{IR}_{\mathcal{I}}(G)}\Big[ \prod_{C \in \mathcal{C}(R)}|MEC(C)|\Big],
	    	\end{equation*}
	    	where $\mathcal{C}(R)$ denotes the set of all chain components of $R$.
		\end{lemma}

	\begin{algorithm}[tb]
		\caption{LazyCount}
		\label{alg:LazyCount}
		\begin{algorithmic}[1]
			\STATE \textbf{Input:} UCCG $G(V,E)$ \\
			\STATE \textbf{Output:} $|MEC(G)|$ \\
			\STATE $CountDP[] \leftarrow$ A storage indexed on $S \subseteq V$ and initiated by 1 if $|S|=1$ and NULL otherwise.
			\STATE \textbf{return $Count(V)$} \newline ------------------------------------------------------------
			\STATE \textbf{function} $Count(S)$
			\IF {$CountDP[S]$ is not $NULL$}
			\STATE return $DP[S]$
			\ENDIF
			\STATE $CountDP[S] \leftarrow 0$
			\STATE $v \leftarrow $ an arbitrary node in $S$
			\STATE $\mathcal{L} \leftarrow LazyIter(G[S], v)$
			\FOR {$R \in \mathcal{L}$}
			\STATE $num \leftarrow 1$
			\FOR {$C(S', E') \in \mathcal{C}(R)$}
			\STATE $num \leftarrow num \times Count(S')$
			\ENDFOR
			\STATE $CountDP[S] \leftarrow CountDP[S] + num$
			\ENDFOR
			\STATE return $CountDP[S]$
		\end{algorithmic}
	\end{algorithm}
	
	 Assume we are given a UCCG $G(V, E)$ and want to calculate $|MEC(G)|$. We first choose an arbitrary node $v \in V$,  set $\mathcal{I} = \{\{v\}\}$, and use $LazyIter$ to find all of the $\mathcal{I}$-essential graphs. Then for each of them, we calculate the number of DAGs inside its corresponding $\mathcal{I}$-MEC by multiplying size of its chain components. As each chain component of an $\mathcal{I}$-essential graph is a UCCG \cite{Hauser_2014}, Lemma \ref{lemma:sigmapi} is applicable on it and we could do the calculation recursively. Finally, we sum up all these values to get $|MEC(G)|$.

	We take advantage of dynamic programming to eliminate repetitive calculations. \citet{Ghassami_2019, DBLP:conf/aaai/TalvitieK19} used a similar idea for observational essential graphs, which we extended to interventional cases.

	The algorithm is presented in Algorithm \ref{alg:LazyCount}. Every time $Count(S)$ is called, it will take $O(1)$ operations if $DP[S]$ is already calculated. Otherwise, it calls $LazyIter$ once which takes $O(2^\Delta(n\Delta + \Delta^3))$ operations, and executes the two for-loops. The outer for-loop is executed at most $2^\Delta$ times, and the inner for-loop is executed at most $n$ times. Calculation of $\mathcal{C}(\mathcal{E})$ could also be done in $O(n\Delta)$ steps. After these calculations, $DP[S]$ will be saved and there is no need to calculate it in later calls. On the other hand, there are at most $2^n$ values for index of $DP$, and therefore the time complexity of Algorithm \ref{alg:LazyCount} is:
		$$O\Big(2^n \big(2^\Delta(n\Delta + \Delta^3) + 2^\Delta (n + n\Delta) \big)\Big) = O(2^n 2^\Delta(n\Delta + \Delta^3)).$$

	\begin{algorithm}[tb]
		\caption{Active Learning by Minimizing $\mathcal{I}$-MEC size}
		\label{alg:activemec}
		\begin{algorithmic}[1]
			\STATE \textbf{Input:} UCCG $G(V,E)$
			\STATE \textbf{Output:} A single-node intervention target $\{v_{opt}\}$
			\STATE $CountDP[] \leftarrow$ A storage indexed on $S \subseteq V$ and initiated by 1 if $|S|=1$ and NULL otherwise.
			\STATE $s_{opt} \leftarrow 0$ 
			\STATE $v_{opt} \leftarrow NULL$

			\FOR {$v \in V$}
			\STATE $\mathcal{L} \leftarrow LazyIter(G,v)$
			\STATE $s_v \leftarrow 0$
			\FOR {$R \in \mathcal{L}$}
			\STATE $mecsize \leftarrow 1$
			\FOR {$C(V', E') \in \mathcal{C}(R)$}
			\STATE $mecsize \leftarrow mecsize \times Count(V')$
			\ENDFOR
			\STATE $s_v \leftarrow max(s_v, mecsize)$
			\ENDFOR
			\IF {$s_v < s_{opt}$}
			\STATE $s_{opt} \leftarrow s_v$
			\STATE $v_{opt} \leftarrow v$
			\ENDIF
			\ENDFOR
			\STATE return    $v_{opt}$
		\end{algorithmic}
	\end{algorithm}
	
	\section{Experiment Design}
	Assume we want to find the best intervention target $I \subseteq V$ in UCCG $G(V,E)$. For experiment design, given an objective function, we need to compare the efficiency of different intervention targets based on it. A common objective function is the size of $\mathcal{I}$-essential graph obtained after intervention \cite{Ghassami_2019}. The smaller the class is, the more information we have gained from the intervention. If we consider the worst-case setting, we have:
	\begin{equation}
	\label{exp:mecsize}
	I_{opt} = \text{arg}\min_{I \subseteq V}\Big(\max_{R \in \mathcal{IR}_{\{I\}}(G)} |MEC(R)|\Big).
	\end{equation}	
	Another objective function used in previous work is the number of directed edges after an intervention \cite{ghassami2018budgeted, Hauser_2014}:
	\begin{equation}
	I_{opt} = \text{arg}\max_{I \subseteq V}\Big(\min_{R \in \mathcal{IR}_{\{I\}}(G)} |Dir(R)|\Big),
	\label{exp:edgecount}
	\end{equation}
	We solve the experiment design problem for both of these objective functions, in both active and passive learning settings.
	\subsection{Active Learning}
	In the active learning, the information obtained from the former interventions can be used to choose the next targets. Similar to the approach taken in \citet{Hauser_2014}, we aim to find the best single-node intervention target in each learning step. We take advantage of $LazyIter$ and $LazyCount$ for this purpose. 
	
	Let $G(V,E)$ be a UCCG. Considering objective function \eqref{exp:mecsize}, we want to find a node $v$ such that intervening on it, minimizes the size of the resulting $\mathcal{I}$-MEC. We first use $LazyIter$ to find the set of all $\mathcal{I}$-essential graphs for different single-node intervention targets. Then, for each $\mathcal{I}$-essential graph $\mathcal{P}^P_v(G)$, we obtain the size of its corresponding $\mathcal{I}$-MEC by multiplying sizes of its chain components. Finally, we use these values to find the optimal intervention target. The description of this algorithm is presented in Algorithm \ref{alg:activemec}. The procedure is almost the same for objective function \eqref{exp:edgecount}. We just need to calculate number of directed edges for each $\mathcal{I}$-essential graph, instead of calculating its $\mathcal{I}$-MEC size.
	
	All of the operations in Algorithm \ref{alg:activemec} could be divided to two parts:
	\begin{itemize}
	    \item Calculating the values of $CountDP[]$ using function $Count()$, which takes at most $O(2^n2^\Delta(n\Delta + \Delta^3))$ operations.
	    \item Iterating over the three for-loops (taking $n$, $2^\Delta$, and $n$ steps respectively), calling $LazyIter$ (taking $O(2^\Delta(n\Delta + \Delta^3))$ operations), and calculating $\mathcal{C}(R)$ (taking $O(n\Delta)$ operations). All of these steps together need $O(n2^\Delta(n\Delta + \Delta^3))$ operations.
	\end{itemize}
	Therefore Algorithm \ref{alg:activemec} calculates the MEC size in at most $O(2^n2^\Delta(n\Delta+\Delta^3)) + O(n2^\Delta(n\Delta + \Delta^3)) = O(2^n2^\Delta(n\Delta+\Delta^3))$ operations. If we want to find the best target with respect to objective function \eqref{exp:edgecount}, there is no need to calculate $CountDP[]$, but all other operations should be executed similarly. Consequently, the time complexity in this case would be $O(n2^\Delta(n\Delta + \Delta^3))$.

	\subsection{Passive Learning}
	
	Let $G(V,E)$ be a UCCG, where each node $v \in V$ is assigned a cost $c_v$. We aim to find a set of $k$ single-node interventions, and therefore our intervention family is of the form $\mathcal{I} = \{\{v_1\}, \{v_2\}, ..., \{v_k\}\}$, similar to the model considered in \citet{ghassami2018budgeted}. 
	Using the following lemma, we break the problem down to smaller subproblems and take advantage of dynamic programming:
	
	\begin{figure*}[tb]
	\centering
		\begin{tikzpicture}[scale=0.95]
		\node[{shape=circle, text=black, minimum size=0.1em}] (v) at  (3.5,-1.5) {a};
		\begin{axis}[
		xlabel=number of edges,
		ylabel=time(s),
		grid=both,
		ytick = {0,500,...,2500},
		minor ytick={0,250,...,2500},
		xtick={150,175,...,300},
		legend style={at={(0.5,0.9)}, legend cell align=left} %
		]
		\addplot[mark=*,cyan] 
		plot coordinates {
		    (150, 2.273839803)
		    (175, 5.957834258)
			(200, 13.33884534)
			(225, 34.69326311)
			(250, 87.35944125)
			(275, 227.8270146)
		};
		\addlegendentry{LazyIter}
		
		\addplot[color=red,mark=triangle]
		plot coordinates {
		    (150, 3.28371497)
		    (175, 8.18834094)
			(200, 51.4337895)
			(225, 250.647167)
			(250, 672.986144)
			(275, 2015.85610)
		};
		\addlegendentry{Hauser2014}
		\end{axis}
		\end{tikzpicture}
		\begin{tikzpicture}[scale=0.95]
		\node[{shape=circle, text=black, minimum size=0.1em}] (u) at  (3.5,-1.5) {b};
		\begin{axis}[
		xlabel=number of edges,
		grid=both,
		minor ytick={0,500,...,2000},
		minor xtick={225,250,...,375},
		legend style={at={(0.5,0.9)}, legend cell align=left} %
		]
		\addplot[color=cyan,mark=*]
		plot coordinates {
            (225, 2.3)
            (250, 5.05)
			(275, 11.55)
	        (300, 30.7)
			(325, 92.81)
			(350, 256.54)
			(375, 889)
		};
		\addlegendentry{LazyCount}
		
		\addplot[mark=triangle,red] plot coordinates {
            (225, 2.55)
            (250, 5.34)
			(275, 15.22)
	        (300, 48.3)
			(325, 141)
			(350, 443)
			(375, 1997)
		};
		\addlegendentry{MemoMAO}
		\end{axis}
		\end{tikzpicture}

		\begin{tikzpicture}[scale=0.7]
		\node[{shape=circle, text=black, minimum size=0.1em}] (u) at  (3.5,-1.5) {c};
		\begin{axis}[
		xlabel=graph order (n),
		ylabel= edge discovery ratio,
		grid = both,
		legend style={at={(0.55,0.55)}, legend cell align=left} %
		]
		\addplot[color=cyan,mark=*]
		plot coordinates {
			
			(10, 0.73) 
			(15, 0.71) 
			(20, 0.70) 
			(25, 0.69) 
			(30, 0.68) 
			(35, 0.68)
		};
		\addlegendentry{Our Method}
		
		\addplot[color=green,mark=square]
		plot coordinates {
			
			(10, 0.45) 
			(15, 0.43) 
			(20, 0.41) 
			(25, 0.40) 
			(30, 0.38) 
			(35, 0.38)
		};
		\addlegendentry{Random}
		
		\addplot[color=red,mark=triangle]
		plot coordinates {	
			
			(10, 0.70) 
			(15, 0.68) 
			(20, 0.64) 
			(25, 0.60) 
			(30, 0.56) 
			(35, 0.53)
			
		};
		\addlegendentry{MaxDegree}
		
		\end{axis}
		\end{tikzpicture}
		\begin{tikzpicture}[scale=0.7]
		\node[{shape=circle, text=black, minimum size=0.1em}] (u) at  (3.5,-1.5) {d};
		\begin{axis}[
		xlabel=edge density (r),
		grid = both,
		legend style={at={(0.7,0.65)}, legend cell align=left} %
		]
		\addplot[color=cyan,mark=*]
		plot coordinates {
			
			(0.1, 0.7312348668280872)
			(0.2, 0.7815126050420168)
			(0.3, 0.776083467094703)
			(0.4, 0.7575030012004802)

		};
		\addlegendentry{Our Method}
		
		\addplot[color=green,mark=square]
		plot coordinates {
			(0.1, 0.3423244552058111)
			(0.2, 0.44243697478991595)
			(0.3, 0.4647833065810594)
			(0.4, 0.478343337334934)
			
		};
		\addlegendentry{Random}
		
		\addplot[color=red,mark=triangle]
		plot coordinates {	
			
			(0.1, 0.5786924939467313)
			(0.2, 0.6866746698679471)
			(0.3, 0.687800963081862)
			(0.4, 0.6656662665066025)
			
		};
		\addlegendentry{MaxDegree}
		
		\end{axis}
		\end{tikzpicture}
		\begin{tikzpicture}[scale=0.7]
		\node[{shape=circle, text=black, minimum size=0.1em}] (u) at  (3.5,-1.5) {e};
		\begin{axis}[
		xlabel=budget (b),
		grid=both,
		minor xtick={1,2,3},
		legend style={at={(0.97,0.28)}, legend cell align=left} %
		]
		\addplot[color=cyan,mark=*]
		plot coordinates {

			(1, 0.5708180708180709)
			(2, 0.7148962148962149)
			(3, 0.7747252747252746)

		};
		\addlegendentry{Our Method}
		
		\addplot[color=green,mark=square]
		plot coordinates {
			(1, 0.22796703296703294)
			(2, 0.35330891330891334)
			(3, 0.4843894993894994)
			
		};
		\addlegendentry{Random}
		
		\addplot[color=red,mark=triangle]
		plot coordinates {	
			
			(1, 0.44017094017094016)
			(2, 0.6043956043956044)
			(3, 0.6788766788766789)
			
		};
		\addlegendentry{MaxDegree}
		
		\end{axis}
		\end{tikzpicture}

		\caption{(a) Comparison between execution times of LazyIter and algorithm in \citet{Hauser_2014} versus number of edges for graphs with 30 nodes. (b) Comparison between execution times of LazyCount and MemoMAO \cite{DBLP:conf/aaai/TalvitieK19} versus number of edges for graphs with 30 nodes. Comparison between edge discovery ratio versus (c) graph order for $b=2$ and $r=0.4$, (d) edge density for $n=35$ and $b=3$, (e) budget for $n=40$ and $r=0.3$} 
		\label{experiments}
	\end{figure*}
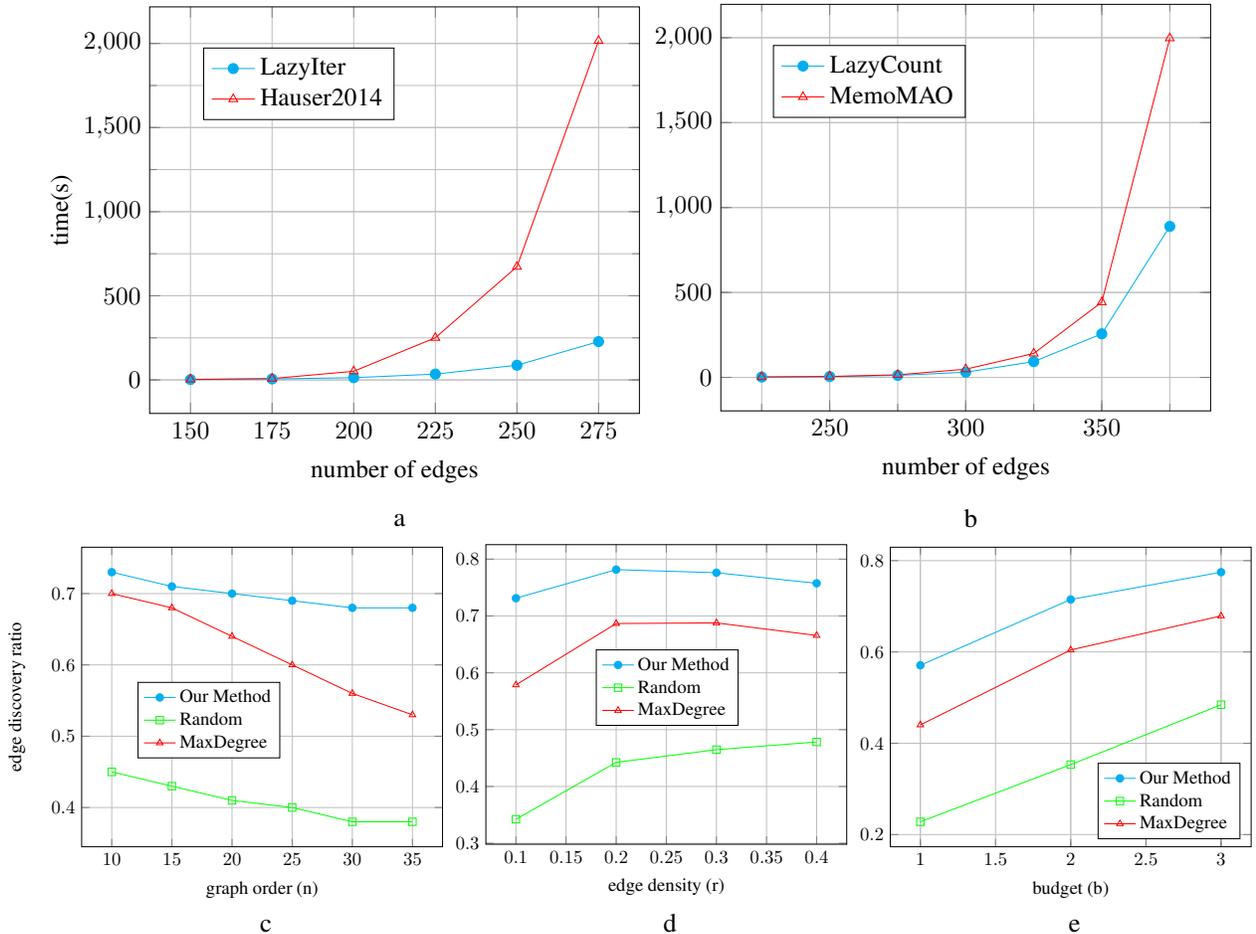
	
	\begin{lemma}
		\label{lemma:limitToComp}
		Let $G(V,E)$ be a UCCG, $\mathcal{I} = \{\{v_1\},\{v_2\}, ..., \{v_k\}\}$ an intervention family, $D$  the ground truth DAG of $G$, and for each chain component $C \in \mathcal{C}(\mathcal{E}_{\{\{v_1\}\}}(G))$, $\{v^{C}_1, v^{C}_2, ..., v^{C}_{m_C}\} \subseteq V$ be the subset of intervention targets which are inside $C$. Then we have:
		$$Dir(\mathcal{E}_{\{\{v_1\}, \{v_2\}, ..., \{v_k\}\}}(D)) = Dir(\mathcal{E}_{\{\{v_1\}\}}(D)) \cup \mathcal{Z},$$
		where 
		$$\mathcal{Z}=\bigcup_{C \in \mathcal{C}(\mathcal{E}_{\{\{v_1\}\}}(D))} Dir(\mathcal{E}_{\{\{v^{C}_1\}, \{v^{C}_2\}, ..., \{v^{C}_{m_C}\}\}}(D[C])).$$
	\end{lemma}
	
	\begin{algorithm}[tb]
		\label{alg:passive}
		\caption{Passive Learning by Maximizing Number of Oriented Edges}
		\begin{algorithmic}[1]
			\STATE \textbf{Input:} UCCG $G(V,E)$, budget $b$, intervention cost for each node $v \in V$ as $cost_v$
			\STATE \textbf{Output:} A single-node intervention target $\{v_{opt}\}$
			\STATE $DP[S][T] \leftarrow$ A storage indexed on $S \subseteq V$ and $T \subseteq S$, and initiated by 0 if $|S|=1$ and NULL otherwise.
			\STATE $best \leftarrow \emptyset$
			\FOR {$T \subseteq V$}
			\IF {$\sum_{x \in T} cost_x \leq budget$}
			\IF {$Calculate(V, best) \leq Calculate(V,T)$}
			\STATE $best \leftarrow  T$
			\ENDIF
			\ENDIF 
			\ENDFOR
			\STATE return $best$ \newline ----------------------------------------------------------------
			\STATE \textbf{function} $Calculate(S,T)$
			\IF {$DP[S][T]$ is not NULL}
			\STATE return $DP[S][T]$
			\ENDIF
			\STATE $DP[S][T] \leftarrow \infty$
			\STATE $v \leftarrow $ an arbitrary member of $T$
			\STATE $\mathcal{L} \leftarrow LazyIter(G, v)$
			\FOR {$R \in \mathcal{L}$}
			\STATE $num \leftarrow |Dir(R)|$
			\FOR {$C(S', E') \in \mathcal{C}(R)$}
			\STATE $num \leftarrow num + Calculate(S',T \cap S')$
			\ENDFOR
			\STATE $DP[S][T] \leftarrow min(DP[S][T], num)$
			\ENDFOR
			\STATE return $DP[S][T]$
		\end{algorithmic}
	\end{algorithm}

	Assume we want to find the optimum intervention target with respect to objective function \eqref{exp:edgecount}. For any $T,S\subseteq V$ where $T=\{v_1, v_2, ..., v_t\}$ and $T\subseteq S$, we define $DP[S][T]$ as follows: 
    \begin{equation}
    \label{exp:DP}
        DP[S][T] = \min_{D \in \textbf{D}(G)} |Dir(\mathcal{E}_{\{\{v_1\}, \{v_2\}, ..., \{v_t\}\}}(D[S])|.
    \end{equation}

	\begin{proposition}
	\label{thm:DPupdate}
	The following equation holds for $DP$ function \eqref{exp:DP}:
	\begin{equation}
	\label{exp:DPupdate}
	\begin{split}
	DP[S][T] =& 
	\\    \min_{R \in \mathcal{IR}_{\{\{v_1\}\}}(G[S])} &\Big\{
	|Dir\Big(R\Big)|
	+
	\sum_{C\in \mathcal{C}(R)}DP\Big[C\Big]\Big[T\cap C\Big]
	\Big\}.
	\end{split}
	\end{equation}
	\end{proposition}

	This proposition suggests that we could select an arbitrary intervention target, iterate over all $\mathcal{I}$-essential graphs in its intervention result space, and find number of directed edge in each case using already-calculated $DP$ values. After finding all $DP[V][T]$ values, we can choose the one which has a cost less than our budget and maximizes number of directed edges. 
	For optimization with respect to objective function \eqref{exp:mecsize}, we can define $DP[S][T]$ as the maximum  size of $\mathcal{I}$-MEC obtained from $G[S]$ after intervening on nodes in $T$. With the similar arguments, we can show that if we substitute $|Dir(R)|$ with $|MEC(R)|$ in equation \eqref{exp:DPupdate}, the resulting equation holds for this new $DP$ array. 
	
	The number of $DP$ elements is $3^n$, as each node is either in $T$, or in $S \backslash T$, or in $V \backslash S$. For calculation of each $DP$ value, $LazyIter$ is called once and then two for-loops are executed, iterating for $2^\Delta$ and $n$ steps respectively. Hence,
		Algorithm \ref{alg:passive} finds the best passive intervention target with respect to objective function \eqref{exp:edgecount} in $O(3^n2^\Delta(n\Delta + \Delta^3))$ operations.

	\section{Experimental Results}
	We compared $LazyIter$ and $LazyCount$ against previous work. The performance of our active learning algorithms depend on these two routines. Our DP-based passive learning algorithm is the first exact algorithm for worst-case experiment design, so we compared it with Random and MaxDegree heuristics. The only related previous work  \citet{Ghassami_2019} is an approximation designed for the average-case passive learning. Their algorithm has a time complexity of $O(kNn^{(\Delta + 1)})$ (where $N$ is the number of sampled DAGs and $k$ is the budget), and is considerably more computationally expensive than our algorithm. However, the results are not comparable as their algorithm does not solve the problem in the worst-case setting. For each test, we generated 100 graphs using the method presented in \citet{he2015counting} and calculated the average test results on them. As we can see in Figure \ref{experiments} (a), $LazyIter$ outperforms \citep{Hauser_2014} in all cases, especially when the graph is dense. We also tested $LazyCount$ against $MemoMAO$, which is the state-of-the-art MEC size calculation algorithm \cite{DBLP:conf/aaai/TalvitieK19}. Even though the difference in execution times is not considerable for sparse graphs, our algorithm performs much better for dense graphs, as seen in Figure \ref{experiments} (b). The main reason for this is that $LazyCount$ requires fewer $DP$ values in its execution. Figures \ref{experiments} (c), (d), and (e) present the discovered edge ratio (the number of edges whose orientations are inferred from experiments to the number of edges in the graph) of the passive learning algorithm versus different graph orders, edge densities (ratio of the number of edges to the maximum possible number of edges), and budgets (number of interevetions), respectively. As the graph order increases, finding the optimal target becomes harder, and therefore the difference between our algorithm and the heuristics becomes more considerable. 
	
	\section{Conclusion}
	We proposed a new method to iterate efficiently over possible $\mathcal{I}$-essential graphs and utilized it to design algorithms for computing MEC size and experiment design for active and passive learning settings. Experimental results showed that the proposed algorithms outperform other related works in terms of time complexity. As a direction of future research, it would be interesting to extend to the proposed algorithms for other objective functions in designing experiments, such as average number of oriented edges. Moreover, one can work on designing algorithms in the passive learning setting where we can intervene on multiple variables in each experiment.

	\bibliography{example_paper}
	\bibliographystyle{icml2020}

	\newpage
	\appendix
	\section{Appendices}
	
	\subsection{Proof of Proposition \ref{thm:onetoone}}
	\begin{proof}
		It could be shown that a DAG $D$ is a member of MEC corresponding to $G$ if and only if it has no v-structures \cite{he2015counting}. Let $v \in V$ be an arbitrary node. For every DAG in the MEC, the parent set of node $v$ is definitely a clique, because a v-structure is formed otherwise. If $D_1$ and $D_2$ be two members of MEC such that $pa_{D_1}(v) = pa_{D_2}(v)$ then $D_1 \sim_{\{\{v\}\}} D_2$, and therefore $D_1$ and $D_2$ are indistinguishable under the single-node intervention target $\{v\}$ \cite{Hauser_2014}.
		So every $\mathcal{E}_{\{\{v\}\}}(D)$ is determined uniquely with $pa_D(v)$.
		Every LexBFS-ordering $\sigma$ on $G$, is also a perfect elimination ordering and if we orient edges of $G$ according to $\sigma$, we get a DAG without v-structures \cite{Hauser_2014}. For an arbitrary clique $P \subseteq ne_G(v)$ in neighbors of $v$, if we orient edge set $E$ according to $LexBFS((P, v, ...), E)$, the resulting DAG $D$ is a member of MEC and $pa_D(v) = P$. This shows that there is a one-to-one correspondence between $\mathcal{E}_{\{\{v\}\}}(D)$s and cliques $P \subseteq ne_G(v)$.
	\end{proof}
	
	\subsection{Proof of Theorem \ref{thm:properties}}
	\begin{proof}
		The proofs of four statements is respectively as follows:
		\begin{itemize}
			\item Every node which is separated from $v$ by $P$ is inside $A_R$, so for every $d \in D_R$ there is a path from $d$ to $v$ in $G[V \backslash A_R]$. Now assume that there is an edge between two arbitrary nodes $a \in A_R$ and $d \in D_R$. As there is a path from $v$ to $d$ in $G[V \backslash A_R]$, and edge $a-d$ is also present in $G[V \backslash A_R]$, there is a path from $v$ to $a$ in $G[V \backslash A_R]$ and therefore $P$ is not an $(a,v)$-separator in $G$, which could not be true.
			\item The cycle $a \rightarrow v \rightarrow b \rightarrow a$ is formed otherwise.
			\item If $a \in C_R$ and the edge be directed as $b\rightarrow a$, the v-structure $v \rightarrow a \leftarrow b$ will be formed. If $a \in P$, let $v, x_1, x_2, ..., x_k, b$ be the shortest path between $v$ and $b$ in $G[\{v\} \cup C_R \cup D_R]$. No two non-consecutive nodes of this path are connected to each other, because we will find a shorter path otherwise. It is also obvious that $x_1 \in C_R$ and therefore $v\rightarrow x_1 \in R$. If $x_1-x_2$ be directed as $x_1 \leftarrow x_2$ in $R$, the v-structure $v \rightarrow x_1 \leftarrow x_2$ will be formed, so $x_1 \rightarrow x_2 \in R$. With a similar arguement, we can say $x_i \rightarrow x_{i+1} \in R$, for $1 \leq i \leq k$, where $x_{k+1} = b$. Therefore $v \rightarrow x_1 \rightarrow x_2 \rightarrow ... \rightarrow x_k \rightarrow b$ is a directed path in $R$. If $b \rightarrow a \in R$, we will have a cycle in $R$ which is impossible, and therefore $a \rightarrow b \in R$.
			\item None of the edges inside $R[A_R \cup P]$ are oriented as a direct result of intervention, so every edge in this subgraph should be oriented using Meek rules. Let $a\rightarrow b$ be the first edge oriented inside $R[A_R \cup P]$, so we have $a,b \in A_R \cup P$. In all of the four Meek rules, there is at least one already oriented edge directed towards one of the two endpoints of the edge which is being oriented. This means that there should exist either an edge $x\rightarrow a \in R$ or and edge $x \rightarrow b \in R$. But this is impossible, because we know that there are no edges directed towards any of the nodes in $A_R \cup P$ in the graph we get after intervention. This means that no Meek rules are applicable for orienting edges in $R[A_R \cup P]$, and this subgraph is undirected.
		\end{itemize}
	\end{proof}
	
	\subsection{Proof of Lemma \ref{lemma:sigmapi}}
	We break the lemma into two smaller lemmas and prove them separately:
	\begin{lemma}
		\label{lemma:sigma}
		Let $G(V,E)$ be a UCCG and $\mathcal{I}$ be an arbitrary intervention family. Then we have: 
		$$ |MEC(G)| = \sum_{R \in \mathcal{IR}_{\mathcal{I}}(G)} |MEC(R)|. $$
	\end{lemma}
	\begin{proof}
		Every DAG $D$ in MEC corresponding to $G$ is exactly in one of the $\mathcal{I}$-essential graphs in $\mathcal{IR}(G)$, based on direction of the edges connected to intervention targets inside that DAG. Therefore, each DAG is exactly counted once in the summation.
	\end{proof}

	\begin{lemma}
		\label{lemma:pi}
		Consider $\mathcal{I}$-essential graph $\mathcal{E}_\mathcal{I}(D)$ of a DAG $D$ and intervention target $\mathcal{I}$. Let $\mathcal{C}(\mathcal{E}_\mathcal{I}(D))$ be the set of all chain components of $\mathcal{E}_\mathcal{I}(D)$. Then we have:
		$$|MEC(\mathcal{E}_\mathcal{I}(D))| = \prod_{C \in \mathcal{C}(\mathcal{E}_\mathcal{I}(D))}|MEC(C)|.$$
	\end{lemma}

	\begin{proof}
    	\citep{Hauser_2014} showed that the direction of edges inside each chain component of an $\mathcal{I}$-essential graph is unrelated to the direction of edges in other components. Therefore edges inside each chain component could be oriented independently, and number of valid orientations of edges in $\mathcal{E}_\mathcal{I}(D)$ (orientations without v-structures) is equal to multiplication of number of valid orientations in each chain component. \citet{he2015counting} proved a similar lemma for observational cases.
	\end{proof}
    	
    Lemma \ref{lemma:sigma} shows that we can calculate the size of MEC represented by $G$ via calculating sizes of $\mathcal{I}$-MECs represented by members of $\mathcal{IR}_{\{\{v\}\}}(G)$. For counting number of DAGs in each of these $\mathcal{I}$-MECs, we use Lemma \ref{lemma:pi}, and therefore the equation in Lemma \ref{lemma:sigmapi} holds.
    
	\subsection{Proof of Lemma \ref{lemma:limitToComp}}
	
	We need these two lemmas for the proof:
	
	\begin{lemma}
		\label{lemma:edgeindependency}
		\cite{ghassami2018budgeted}
		For any DAG $D(V,E)$ and sets $I_1, I_2 \subseteq V$, we have: $$Dir(\mathcal{E}_{\{I_1\cup I_2\}}(D)) =  Dir(\mathcal{E}_{\{I_1\}}(D))\cup Dir(\mathcal{E}_{\{I_2\}}(D)).$$
	\end{lemma}
	
	\begin{lemma}
		\label{lemma:comppnentinterventionindependency}
		\cite{Hauser_2014} Consider an $\mathcal{I}$-essential graph of some DAG $D$, and let $C \in \mathcal{C}(\mathcal{E}_\mathcal{I}(D))$ be one of its chain components. Let $I \subseteq V, I \notin \mathcal{I}$ be another intervention target. Then we have:
		$$\mathcal{E}_{\mathcal{I} \cup \{I\}}(D)[C] = \mathcal{E}_{\{\emptyset, I \cap V'\}}(D[C])$$
	\end{lemma}
	
	Now we prove Lemma \ref{lemma:limitToComp}.
	\begin{proof}
		Using Lemma \ref{lemma:comppnentinterventionindependency} we can say:
		$$\mathcal{E}_{\{\{v_1\}, \{v_2\}, ..., \{v_k\}\}}(D)[C] =$$
		$$\mathcal{E}_{\{\{v_1\}\} \cup \{\{v_2\}, ..., \{v_k\}\}}(D)[C]=$$
		$$\mathcal{E}_{\{ \emptyset, \{\{v_2\}, ..., \{v_k\}\} \cap V'\}}(D[C])=$$
		$$\mathcal{E}_{\{\{\{v_2\}, ..., \{v_k\}\} \cap V'\}}(D[C])=$$
		$$\mathcal{E}_{\{\{v^{C}_1\}, ..., \{v^{C}_m\}\}}(D[C])$$
		Where the equality between third and fourth lines comes from the fact that we already know the observational essential graph of the chain component, as we are given the UCCG.
		Using Lemma \ref{lemma:edgeindependency}, we have: 
		$$Dir(\mathcal{E}_{\{\{v_1\}, \{v_2\}, ..., \{v_k\}\}}(D)) $$
		$$=Dir(\mathcal{E}_{\{\{v_1\}\}}(D)) \cup Dir(\mathcal{E}_{\{\{v_2\}, \{v_3\}, ..., \{v_k\}\}}(D))$$
		But as we mentioned earlier, direction of edges inside one chain component gives us no information about direction of edges in other chain components. 

		We can say:
		\begin{equation*}
		\begin{split}
		Dir(R_1)
		\cup&  Dir(\mathcal{E}_{\{\{v_2\}, \{v_3\}, ..., \{v_k\}\}}(D))   
		\\=Dir(R_1)& \bigcup_{C \in \mathcal{C}(R_1)} Dir(\mathcal{E}_{\{\{v_2\}, \{v_3\}, ..., \{v_k\}\}}(D)[C]) 
		\\=Dir(R_1)&\bigcup_{C \in \mathcal{C}(R_1)} Dir(\mathcal{E}_{\{\{v^C_1\}, \{v^C_2\}, ..., \{v^i_{m_C}\}\}}(D)[C]).
		\end{split}
		\end{equation*}

	\end{proof}
	\subsection{Proof of Proposition \ref{thm:DPupdate}}
	\begin{proof}
	We know that every valid orientation of all undirected edges in all of the chain components gives us a DAG in the $\mathcal{I}$-MEC $\mathcal{E}_{\{\{v_1\}\}}(D)$. Moreover we know that the minimum value of $|Dir(\mathcal{E}_{\{\{v^C_1\}, \{v^C_2\}, ..., \{v^C_{m_C}\}\}}(D)[C])|$ is $DP[C][\{v^C_1, v^C_2, ..., v^C_{m_C}\}] = DP[C][T \cap C]$. As chain components have distinct edges sets, we have:
	
	\begin{equation*}
	\begin{split}
	\Big|&\bigcup_{C \in \mathcal{C}(\mathcal{E}_{\{\{v_1\}\}}(D))} Dir(\mathcal{E}_{\{\{v^{C}_1\}, \{v^{C}_2\}, ..., \{v^{C}_{m_C}\}\}}(D[C]))\Big| =
	\\&\sum_{C \in \mathcal{C}(\mathcal{E}_{\{\{v_1\}\}}(D))} \Big|Dir(\mathcal{E}_{\{\{v^{C}_1\}, \{v^{C}_2\}, ..., \{v^{C}_{m_C}\}\}}(D[C]))\Big|.
	\end{split}
	\end{equation*}
	
	Lemma \ref{lemma:limitToComp} implies that for counting number of directed edges in each $\mathcal{I}$-essential graph, we could consider each component independently and therefore the minimum number of directed edges for each chain component can be found via $DP$ values. We can iterate over all possible $\mathcal{E}_{\{\{v_1\}\}}(D)$s and use $DP$ values to find the minimum number of directed edges for each case. This means $DP[V][T]$ could be calculated by the recursive formula \eqref{exp:DPupdate}.
	\end{proof}
	
\end{document}